\documentclass[aoas]{imsart}

\RequirePackage{amsthm,amsmath,amsfonts,amssymb}
\RequirePackage[authoryear]{natbib}
\RequirePackage[colorlinks,citecolor=blue,urlcolor=blue]{hyperref}
\RequirePackage{graphicx}
\usepackage{amsmath}
\usepackage{amsfonts}
\usepackage{amsthm}
\usepackage{booktabs}
\usepackage{xcolor}
\usepackage{subcaption}
\usepackage{graphicx,psfrag,epsf}
\usepackage{enumerate}
\usepackage{algorithm} 
\usepackage{url}
\usepackage{algpseudocode} 
\usepackage{natbib}
\usepackage{todonotes}
\usepackage{makecell}
\usepackage{url}

\startlocaldefs
\newcommand{\R}{\mathbb{R}}
\newcommand{\X}{\mathcal{X}}
\newcommand{\T}{\mathcal{T}}
\newcommand{\mi}{\mathrm{i}}

\algrenewcommand\algorithmicrequire{\textbf{Input:}}
\algrenewcommand\algorithmicensure{\textbf{Output:}}
\theoremstyle{remark}
\newtheorem{definition}{Definition}
\newtheorem{theorem}{Theorem}

\newtheorem{lemma}{Lemma}
\newtheorem{corollary}{Corollary}

\endlocaldefs

\begin{document}

\begin{frontmatter}

\title{Assessing the Distributional Fidelity of Synthetic Chest X-rays using the Embedded Characteristic Score}

\runtitle{Embedded Characteristic Score to Evaluate Synthetic Chest X-rays}
 
\begin{aug}

\author[A]{\fnms{Edric}~\snm{Tam} \ead[label=e1]{edrictam@stanford.edu}}
\and
\author[A, B]{\fnms{Barbara}~\snm{E. Engelhardt}\ead[label=e2]{bengelhardt@stanford.edu}}

\address[A]{Department of Biomedical Data Science, Stanford University\printead[presep={ ,\ }]{e1}}

\address[B]{Gladstone Institutes\printead[presep={,\ }]{e2}}
\end{aug}
\begin{abstract}
Chest X-ray (CXR) images are among the most commonly used diagnostic imaging modalities in clinical practice. Stringent privacy constraints often limit the public dissemination of patient CXR images, contributing to the increasing use of synthetic images produced by deep generative models for data sharing and training machine learning models. Given the high-stakes downstream applications of CXR images, it is crucial to evaluate how faithfully synthetic images reflect the underlying target distribution. We propose the embedded characteristic score (ECS), a flexible evaluation procedure that compares synthetic and patient CXR samples through characteristic function transforms of feature embeddings. The choice of embedding can be tailored to the clinical or scientific context of interest. By leveraging the behavior of characteristic functions near the origin, ECS is sensitive to differences in higher moments and distribution tails, aspects that are often overlooked by commonly used evaluation metrics such as the Fr\'echet Inception Distance (FID). We establish theoretical properties of ECS and describe a calibration strategy based on a simple resampling procedure. We compare the empirical performance of ECS against FID via simulations and standard benchmark imaging datasets. Assessing synthetic CXR images with ECS uncovers clinically relevant distributional discrepancies relative to patient CXR images. These results highlight the importance of reliable evaluation of synthetic data that inform high-stakes decisions. 

\end{abstract}

\begin{keyword}
\kwd{Evaluation}
\kwd{Generative Models}
\kwd{Chest X-ray}
\kwd{FID}
\kwd{Medical Images}
\end{keyword}

\end{frontmatter}

\section{Introduction}
\label{sec:intro}

Generative modeling approaches are increasingly used in medical imaging to produce synthetic images when access to patient data is limited by privacy, regulatory or logistical constraints \citep{arora2025urgent}. Advances in deep generative modeling architectures, including the variational autoencoder (VAEs) \citep{kingma2014auto}, generative adversarial networks (GANs) \citep{goodfellow2014generative}, and diffusion models \citep{ho2020denoising, song2021score} have enabled progressively more realistic images to be synthesized. This contributes to the growing use of synthetic images to train machine learning models \citep{man2022review, fan2024scaling, mumuni2024survey}, particularly in medical imaging applications \citep{koetzier2024generating, coyner2022synthetic}. 

Synthetic medical images that do not faithfully reflect the underlying target image distribution can propagate hidden biases \citep{norori2021addressing} or fail to capture less common pathological features, potentially affecting downstream model performance or clinical interpretation \citep{pencina2020prediction}. Given the central role that medical images play in informing clinical decision making \citep{giuffre2023harnessing}, rigorous evaluation of the fidelity of synthetic medical images is essential. 

In this work, we focus on chest X-ray (CXR) imaging, which plays a central role in the diagnosis of a wide range of cardiopulmonary conditions, and is used in both acute and routine clinical settings \citep{raoof2012interpretation}. With an estimated 129 million chest X-rays performed in the United States alone in 2006 \citep{mettler2009radiologic}, this imaging modality has attracted substantial interest in deep learning and in synthetic image generation \citep{ccalli2021deep, truhn2025synthetic}, making it a natural and important testbed for evaluation. 

Currently, the most commonly used quantitative metric for comparing patient versus synthetic medical images is the Fr\'echet Inception Distance (FID) and related variants \citep{woodland2024feature, konz2026frechet, heusel2017gans}. The FID first extracts embeddings of patient and synthetic medical images using a pretrained Inception v3 deep neural network \citep{szegedy2016rethinking}. It then compares the distributions of patient image embeddings and synthetic image embeddings using a Fr\'echet distance under a multivariate Gaussian assumption. Another important approach for evaluating synthetic medical images is via human experts, where radiologists are asked to distinguish between synthetic and real images in so-called visual Turing assessments \citep{pan20232d}. We give a brief overview of evaluation metrics for synthetic images in Section \ref{related}. 

Recent evidence suggests that generative AI models have reached a level of maturity where synthetic medical images are increasingly perceptually indistinguishable from authentic patient images, even to expert radiologists on visual Turing tests \citep{koetzier2024generating, pan20232d, jang2023image}. However, visual realism does not guarantee distributional fidelity (Figure~\ref{fig:cxr_panel}). Synthetic datasets might fail to encompass the full spectrum of biological variability found in patient-derived data. Because critical diagnostic information, such as pathological markers or rare clinical phenotypes, often resides in the extremes of the data distribution, evaluation methods sensitive to these subtle discrepancies are essential for improving clinical safety.
 
\begin{figure}[t]
\centering
\begin{subfigure}[t]{0.48\textwidth}
    \centering
    \includegraphics[width=\textwidth]{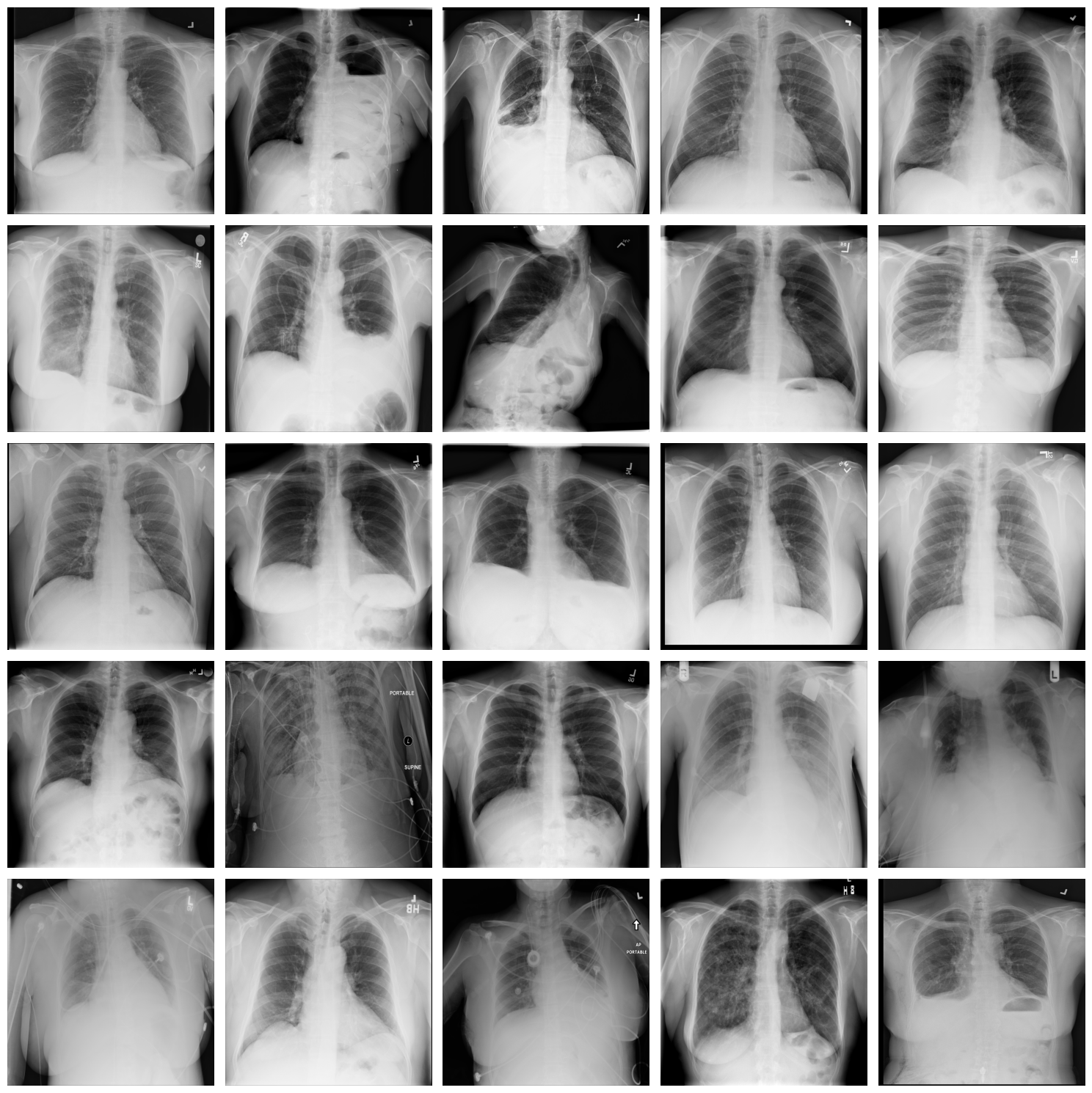}
    \caption{Authentic chest X-ray images.}
    \label{fig:cxr_real}
\end{subfigure}
\hfill
\begin{subfigure}[t]{0.48\textwidth}
    \centering
    \includegraphics[width=\textwidth]{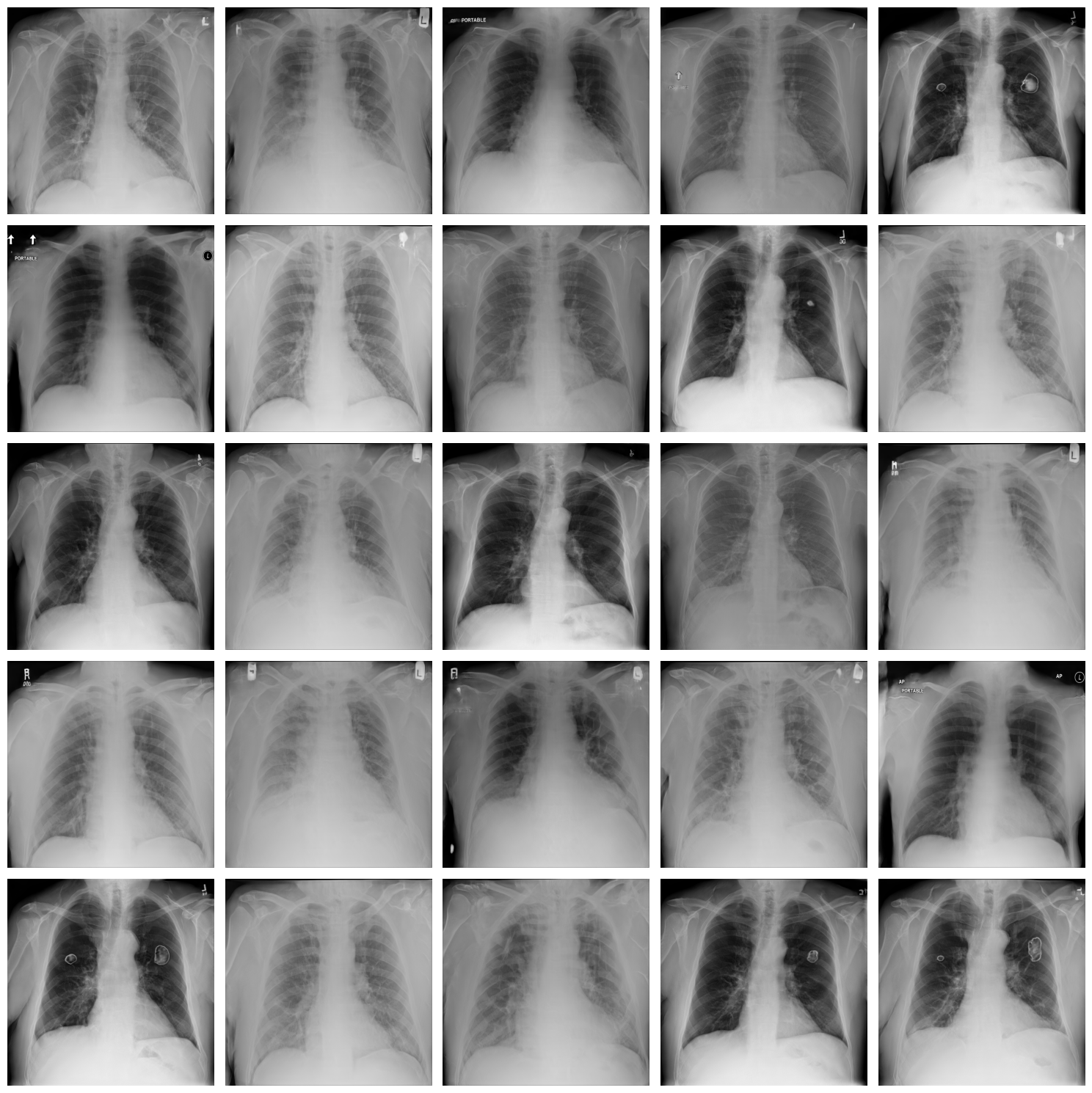}
    \caption{Synthetic chest X-ray images.}
    \label{fig:cxr_synth}
\end{subfigure}
\caption{\textbf{Comparison of authentic and synthetic chest X-ray images.} 
A panel of representative authentic images from the ChestX-ray14 dataset (left) and synthetic images (right) generated by a GAN pretrained on the same ChestX-ray14 dataset.}
\label{fig:cxr_panel}
\end{figure}

Here, we argue that current evaluation metrics are insufficient to account for the distributional mismatches between synthetic and patient CXR images. While approaches based on human evaluation may be appropriate for evaluating the visual realism of individual synthetic images \citep{otani2023toward}, human evaluators are generally not well suited to judging distributional agreement, given well-documented cognitive limitations in assessing randomness \citep{tversky1974judgment, williams2008people}. On the other hand, current quantitative metrics often fail to properly account for differences between the synthetic and target distributions. For example, the FID relies on normality approximations that do not hold in practice and obscure important discrepancies in the higher moments and tails. 

We propose the embedded characteristic score (ECS) as a complementary approach that gauges the distributional match between patient and synthetic CXR images. Rather than assessing visual realism, ECS aims to capture discrepancies in higher moments and tails that standard evaluation metrics based on normality assumptions may overlook. The ECS consists of two main components: feature embeddings of the raw images, and characteristic function transforms of these features. ECS allows the choice of feature embedding to be specified by the practitioner, enabling evaluation to be aligned with the scientific or clinical context of interest. Once the feature embedding is computed, the characteristic function of each feature is then estimated at designated points near the origin and used to compare the synthetic and patient images. We focus on characteristic functions near the origin because they are known to capture tail and moment information while remaining stable to estimate, even under heavy-tailed data. We also provide a simple resampling-based calibration procedure that helps practitioners interpret the magnitude of ECS in a principled way.

Before turning to our motivating chest X-ray application, we assess the empirical behavior of ECS in two settings. First, we conduct a simulation study to gauge the ability of ECS to faithfully capture tail and moment differences between known probability distributions. Second, we compare ECS against FID on benchmark image datasets to facilitate a direct comparison in a standardized setting. We then apply ECS to evaluate synthetic chest X-ray images from an advanced generative model \citep{osuala2023medigan} pretrained on a large-scale chest X-ray dataset \citep{wang2017chestx}. We use an embedding map that is tailored to include interpretable and clinically relevant features for CXR images. Examining ECS along these features reveals distributional discrepancies between patient and synthetic CXR images.

Our goal is not to replace existing evaluation approaches for CXR data such as visual Turing assessments. Rather, ECS is intended as a complementary diagnostic tool that provides additional insights into distributional mismatches that may not be captured by commonly used metrics. Used alongside existing methods, ECS supports a more nuanced assessment of the fidelity of synthetic CXR images.

While it is tempting to frame the evaluation of synthetic CXR images within a hypothesis testing paradigm, binary accept–reject decisions and formal hypothesis tests are not the primary focus of this work. Because synthetic data are inherently approximations to observed patient data, the central question is not to test whether differences exist, but rather whether such differences are practically relevant on clinically important features that influence downstream interpretation and decision making.

\section{Related Work}
\label{related}
 
A variety of metrics have been proposed for assessing the quality of synthetic images in the machine learning and computer vision literature. We can broadly divide such metrics into two categories: \textit{perceptual realism}, which evaluates the individual properties of generated images, and \textit{distributional comparisons}, which evaluate the collective properties of generated images, such as variability and overall distributional agreement. 

Human evaluation of generated images, often treated as the gold standard for image generative model evaluation \citep{zhou2019hype}, falls into the former category. The core idea is to use human evaluators to distinguish synthetic images from real ones \citep{denton2015deep, rossler2019faceforensics++, pan20232d}. In the medical imaging literature, this often takes the form of visual Turing assessments, where radiologists are shown synthetic and patient images and asked to identify which ones are generated \citep{pan20232d}. Recently, a more refined metric called Human eye perceptual evaluation (HYPE) \citep{zhou2019hype} that incorporates psychophysical principles for human evaluation has been proposed. The HYPE approach is to display images to human evaluators with adaptive time constraints, so that the perceptual threshold for distinguishing real images from synthetic ones can be estimated. 

While human evaluation can be effective at gauging the realism of individual images, they are less effective at assessing the collective and distributional properties of synthetic images. This is due to well known cognitive limitations of humans when handling randomness \citep{tversky1974judgment,williams2008people}. The most common approach for evaluating synthetic images in the \textit{distributional comparisons} category is the Fr\'echet Inception Distance \citep{heusel2017gans}. Here, an embedding vector is computed for each image using the Inception v3 neural network. The embedding vectors that correspond to real images are then compared to those that correspond to synthetic images under normality assumptions using a Fr\'echet distance. Historically, the Inception Score (IS) \citep{salimans2016improved} has also been used to evaluate synthetic images.  
Numerous other evaluation approaches have been proposed based on different considerations, such as density estimation \citep{goodfellow2014generative}, kernel techniques \citep{binkowski2018demystifying}, and information theory approaches \citep{jalali2024information}. For more systematic reviews and studies, see \citep{xu2018empirical, theis2015note, betzalel2024evaluation, stein2024exposing}. 

In medical imaging, visual Turing assessments are considered the gold standard for evaluating the perceptual realism of synthetic images \citep{koetzier2024generating, pan20232d}. However, it has been noted that visual assessment alone is insufficient for evaluating synthetic medical images \citep{yamamoto2023visual}. For assessing collective or distributional properties, the Fr\'echet Inception Distance (FID) has emerged as the de facto standard \citep{koetzier2024generating}. 
A variant based on the Fr\'echet distance have been proposed for evaluating synthetic echocardiograms \citep{abdusalomov2023evaluating}. A recent review \citep{zamzmi2025scorecard} highlights several limitations of existing medical image evaluation standards.

Existing evaluation metrics for synthetic medical images largely focus on perceptual realism (visual Turing assessments) or low-order distributional properties (Fr\'echet distances), and are often insensitive to differences in higher moments and tail behavior. In this work, our goal is to design a complementary metric that can provide insight into the distributional fidelity of synthetic CXR images.
  
\section{Embedded Characteristic Score (ECS)}
\label{sec:meth}
In this section, we give a precise definition of the embedded characteristic score, as well as related theoretical properties that motivate this definition. 

\subsection{Background and Setup}
Consider the chest X-ray images as coming from some sample space $\X$ that is endowed with some $\sigma$-algebra $\mathcal{F}$ of events and some collection of corresponding probability measures $\mathcal{P}$.

Given a target distribution $P \in \mathcal{P}$, represented by a collection of independent patient images $X_1, \cdots, X_n \in \X$, and a learned distribution $\tilde{P}\in \mathcal{P}$, represented by a collection of synthetic images $\tilde{X}_1, \cdots, \tilde{X}_m \in \X$, the goal is to assess the differences between $P$ and $\tilde{P}$. We write $X$ and $\tilde{X}$ as generic random objects that are distributed according to $P$ and $\tilde{P}$ respectively. 

Since the data under consideration are images, the sample space $\X$ is often not endowed with any obvious algebraic and metric structures to enable direct comparisons. For evaluation, instead of directly operating in $\X$, it is natural to consider Euclidean representations of such data. This is commonly achieved by considering an embedding map $f: \X \to \R^p$. Write $f_1, \cdots, f_p: \X \to \R$ for the component functions of $f$, which can be interpreted as feature maps. It is crucial that the chosen embedding map extracts relevant and informative features, since the downstream evaluation operates exclusively in this embedded space. 

In chest X-ray images, the clinically relevant features that guide evaluation often include the positions and sizes of various anatomical structures as well as pathological markers \citep{broder2011imaging}. We discuss specific choice of feature embedding maps for CXR images in section \ref{sec:practical}.  In the context of natural images, $f$ is often chosen to be the last layer representation of a pretrained deep neural network, such as the Inception v3 model \citep{heusel2017gans}. When necessary, we append an additional subscript to indicate the choice of embedding map; for example, $f_{p,inception}$ denotes the $p$th coordinate of the Inception embedding. We use $\rho$ as an index that runs through $1 \cdots p$. 

It is natural to compare the features $f(X)$ and $f(\tilde{X})$. The Fr\'echet Inception Distance (FID) \citep{heusel2017gans}, the most commonly adopted evaluation metric, assumes that $f_{inception}(X)$ and $f_{inception}(\tilde{X})$ both follow multivariate normal distributions. Concretely, given the embedding map $f_{\mathrm{inception}} : \X \to \R^p$, consider the random vectors
$$
Z = f_{\mathrm{inception}}(X), 
\qquad 
\tilde Z = f_{\mathrm{inception}}(\tilde X),
$$
where $X$ and $\tilde X$ denote patient and synthetic images, respectively. FID compares the distributions of $Z$ and $\tilde Z$ by approximating each with a multivariate Gaussian distribution in $\R^p$. Let $(\mu, \Sigma)$ and $(\tilde \mu, \tilde \Sigma)$ denote the mean and covariance of $Z$ and $ \tilde Z$. The FID is defined as
$$
\mathrm{FID}(P, \tilde P)
= \|\mu - \tilde \mu\|_2^2
+ \mathrm{tr}\!\left(\Sigma + \tilde \Sigma 
- 2(\Sigma^{1/2}\tilde \Sigma \Sigma^{1/2})^{1/2}\right),
$$
which corresponds to the squared 2-Wasserstein distance between the Gaussian approximations. In practice, $\mu$, $\tilde \mu$, $\Sigma$, and $\tilde \Sigma$ are unknown and are replaced by their empirical counterparts computed from the samples $\{Z_i\}_{i=1}^n$ and $\{\tilde Z_j\}_{j=1}^m$ during estimation.

Thus, FID evaluates discrepancies between embedded distributions through differences in their first and second moments under a Gaussian assumption. As shown in prior work \citep{jayasumana2024rethinking} and in our simulations in Section \ref{sec:empirical}, such assumptions are not always appropriate and, moreover, they ignore valuable tail and higher moment information. 

Contrary to the approach that FID takes, we would like to compare the tails $P(|f_\rho(X)| > s)$ and $P(|f_\rho(\tilde{X})| > s)$ as well as $\ell$\textsuperscript{th} moments $E(f_\rho^{\ell}(X))$ and $E(f_\rho^{\ell}(\tilde{X}))$ of the features, where $\rho$ is an index that runs through $1, \cdots, p$. Directly estimating the tails $P(|f_\rho(X)| > s)$ and moments $E(f_\rho^{\ell}(X))$ can be problematic. The event $\{|f_\rho(X)| > s\}$ occurs rarely, so the natural estimator $\frac{1}{n} \sum_{j = 1}^n 1_{{f_\rho(X_j) > s}}$, where $1$ denotes the indicator function, may exhibit prohibitively high variance. For distributions with heavy tails, higher order moments $E(f_\rho^{\ell}(X))$ might not exist, in which case the natural estimator $\frac{1}{n} \sum_{j = 1}^n f_\rho^{\ell}(X_j)$ does not converge to any meaningful quantity. 

Observing that the smoothness at the origin of the characteristic function of a real-valued random variable encodes valuable information about both the tail and moments of a distribution, we propose to use the characteristic function as a more stable proxy to compare feature moments and tails. 

\subsection{Embedded Characteristic Score (ECS) Definition}

We define the embedded characteristic score (ECS) here. 

\begin{definition}
Given an embedding map $f:\X \to \R^p$, and some $T > 0$, we define the \textit{embedded characteristic score} as $r_{f, T}(P, \tilde{P}) = \frac{1}{pT}\sum_{\rho = 1}^p||E\exp(\mi T f_\rho(X)) - E\exp(\mi T f_\rho(\tilde{X}))||_2$. 
\end{definition}

Here, $T$ is taken as a positive number that is close to $0$. We suggest computing $r_{f, T}$ for different values of $T$ close to $0$, which leads to a comprehensive comparison in practice. When the context is clear, we often suppress the subscript $f$ and $\T$. When independent samples $X_1, \cdots, X_n$ from $P$ and $\tilde{X}_1, \cdots, \tilde{X}_m$ from $\tilde{P}$ are available, a natural estimator for the embedded characteristic score is  
\begin{align*}
    \hat{r}(P, \tilde{P}) = \frac{1}{pT}\sum_{\rho = 1}^p||\frac{1}{n}\sum_{i = 1}^n\exp(\mi T f_\rho(X_i)) - \frac{1}{m}\sum_{j = 1}^m\exp(\mi T f_\rho(\tilde{X}_j))||_2\;.
\end{align*} 

\begin{algorithm}[t]
\caption{Estimating the embedded characteristic score (ECS).}
\label{alg:ecs}
\begin{algorithmic}[1]
\Require patient images $\{X_i\}_{i=1}^n \sim_{\text{i.i.d.}} P$, synthetic images $\{\tilde X_j\}_{j=1}^m \sim_{\text{i.i.d.}} \tilde P$, embedding $f:\mathcal{X}\to\mathbb{R}^p$, frequency $T>0$
\Ensure $\hat r_{f,T}(P,\tilde P)$
\For{$\rho = 1,2,\ldots,p$}
  \State $J_\rho \gets \frac{1}{n}\sum_{i=1}^n \exp\!\bigl(\mi\, T\, f_\rho(X_i)\bigr)$
  \State $K_\rho \gets \frac{1}{m}\sum_{j=1}^m \exp\!\bigl(\mi\, T\, f_\rho(\tilde X_j)\bigr)$
  \State $Q_\rho \gets \left|J_\rho - K_\rho\right|$
\EndFor
\State \Return $\frac{1}{pT}\sum_{\rho=1}^p Q_\rho$
\end{algorithmic}
\end{algorithm}

We give a theorem that shows the consistency of $\hat{r}_{f,T}(P, \tilde{P})$. 

\begin{theorem}\label{thm:converge}
As $n, m \to \infty$, $\hat{r}_{f, T}(P, \tilde{P})$ converges to $r_{f, T}(P, \tilde{P})$ in probability. 
\end{theorem}

\begin{proof}
 The proof is given in the Appendix.
\end{proof}
We now show that $r_{f, T}$ is a pseudometric. 

\begin{theorem}\label{thm:pseudo}
    The embedded characteristic score $r_{f, T}(\cdot, \cdot)$ is a pseudometric on the space of probability measures on $\X$. 
\end{theorem}
\begin{proof}
 The proof is given in the Appendix. 
\end{proof} 

\subsection{Characteristic function around the origin}
We now review related facts regarding the characteristic function, and we provide results that motivate estimating the characteristic function's values around the origin. See \cite{lukacs1970characteristic} for an authoritative treatment of characteristic functions.  

Any real-valued random variable $Y$ is associated with a unique characteristic function $\phi_Y(t) = E(\exp(\mi tY))$. $\phi_Y$ is uniformly continuous in $t$ and exhibits the bound $|\phi_Y(t)| \leq 1$, with $\phi_Y(0)$ always equal to $1$. $\phi_Y(-t)$ is given by the complex conjugate of $\phi_Y(t)$. It is also well known that the smoothness of $\phi_Y(t)$ around $0$ encodes information about the moments of $Y$. In particular, if $E(Y^\ell)$ exists for some integer $\ell > 0$, then $E(Y^\ell) = i^{-\ell} \phi^{(\ell)}_Y(0)$, where the superscript denotes the $\ell$\textsuperscript{th} derivative. 

We now relate the tail behavior of $Y$ to its characteristic function. We first note the following lemma from \cite{durrett2019probability}. 

\begin{lemma}[\cite{durrett2019probability}]
\label{durrett_tail}
For any real-valued random variable $Y$ and any $u > 0$, $P(|Y| > 2/u) \leq \frac{1}{u} \int_{-u}^u (1 - \phi_Y(t)) dt$, where $\phi_Y(t)$ is the characteristic function of $Y$. While $\phi_Y(t)$ is a complex-valued function, the quantity $\int_{-u}^u (1 - \phi_Y(t)) dt$ is always real-valued. 
\end{lemma}
\begin{proof}
    This inequality is due to 3.3.1 of \cite{durrett2019probability}, which is proved in Theorem 3.3.17 of the same reference. 
\end{proof}
An immediate consequence of this lemma is the following corollary concerning the behavior of the characteristic function around the origin. 

\begin{corollary}\label{origin_tail}
The random variable $Y$ satisfies the tail inequality 
$$P(|Y| > 2s) \leq s \int_{-1/s}^{1/s} (\phi_{Y}(0) - \phi_{Y}(t)) dt$$
for any $s > 0$.  
\end{corollary}

Corollary \ref{origin_tail} suggests that the tail of $Y$ can be bounded by a quantity that represents the smoothness of $\phi_Y$ around the origin. To understand this intuitively, consider an approximation of this integral by the elementary trapezoidal rule, where the quantity $s\int_{-1/s}^{1/s} ( \phi_Y(0) - \phi_Y(t))dt $ is approximated by $2(1 - \text{Re}(\phi_Y(1/s))$. Under mild smoothness conditions, the absolute error of this approximation can be bounded by $C|\frac{1}{s}|^3$, where $C$ here is a positive constant that depends on the value of the second derivative of $\phi_Y$ at a location between $-1/s$ and $1/s$. As $s$ becomes larger, the approximation error becomes smaller. Hence the value of the characteristic function at small $1/s$ provides useful information about the tail $P(|Y| > 2s)$. 

The characteristic function of $Y$ around the origin also gives information about the moments of $Y$. Assume that the $\ell$\textsuperscript{th} moment $E(Y^\ell)$ exists. Taking the Taylor series expansion of $\phi_Y$ around the origin yields \citep{durrett2019probability}
$$\phi_Y(t) = \sum_{v = 0}^\ell\frac{(\mi t)^vE(Y^v)}{v!} + o(t^\ell),$$
where $o(t^\ell)/t^\ell \to 0$ as $t \to 0$. 


In the context of ECS, large differences in the $\ell$\textsuperscript{th} moment of $f_\rho(X)$ and $f_\rho(\tilde{X})$ might be reflected in the difference between their corresponding characteristic functions $\phi_{f_\rho(X)}$ and $\phi_{f_\rho(\tilde{X})}$ near the origin. Characteristic functions offer a more robust alternative to direct moment estimation. The natural estimator $\frac{1}{n}\sum_{i = 1}^n\exp(\mi Tf_\rho(X_i))$ always converges, regardless of whether higher order moments exist. This also avoids the rare data problem of the tail estimator $\frac{1}{n} \sum_{j = 1}^n 1_{{f_\rho(X_j) > s}}$.

When the random variable $Y$ admits a probability density under the Lebesgue measure, the characteristic function is the Fourier transform of the probability density up to a sign reversal. The value of the characteristic function at point $T$ then corresponds to the amplitude associated with the frequency $T$. While estimating $\phi_Y$ across the entire frequency domain is challenging, estimating $\phi_Y$ at a designated point $T$ is a more tractable task. The Fourier uncertainty principles \citep{hogan2005time} imply, roughly, that functions that are well dispersed in the origin domain have Fourier transforms that are highly localized in the frequency domain. This implies that low-frequency behavior captures coarse, global features of a distribution, which can indirectly reflect tail properties. This motivates evaluating characteristic functions at values of $T$ near the origin in the ECS.

\section{Practical Considerations in ECS}
\label{sec:practical}
There are several practical considerations when using the ECS approach to evaluate synthetic images. First, the practitioner must specify (i) the choice of the embedding map $f$ and (ii) the frequency parameter $T$ at which the characteristic function is evaluated. These choices determine the aspects of the data distribution that are emphasized by the resulting score and should therefore be guided by the scientific or application context.

Second, because ECS is a continuous score without an intrinsic scale, its magnitude can be difficult to interpret in isolation, particularly when comparing across datasets or against other evaluation metrics that operate on different numerical scales. To address this, we propose a simple resampling procedure to calibrate ECS that provides a data-driven reference scale for interpreting the magnitude of observed discrepancies.

\subsection{Choice of Embeddings} 
The embedding map $f:\mathcal X \to \mathbb R^p$ should provide a informative feature representation of the underlying images. Such an embedding map is most often provided by a pre-trained neural network, given their strong performance across a broad range of image-related tasks. For many computer vision applications involving natural images, the Inception embedding map is the commonly used default. 

In medical imaging settings, however, the Inception embedding may not be the most appropriate choice. First, the Inception network is trained on generic natural images rather than on modality-specific medical data. Second, the resulting Inception embeddings are not easily interpretable, making it difficult to extract clinically meaningful insights that could guide model development or comparison. 

In the case of chest X-ray imaging, radiologists routinely rely on a number of clinically meaningful features to inform diagnosis. Examples include asymmetry measures such as the left–right lung ratio, which can indicate conditions such as pneumonia or atelectasis \citep{tulo2023evaluation, armato1999computerized}, as well as the cardiothoracic ratio (CTR), defined as the ratio of the maximal horizontal cardiac diameter to the thoracic diameter, and widely used as a screening metric for cardiomegaly and heart failure \citep{brakohiapa2017radiographic, kizilgoz2025update}. These quantities can be naturally computed from the outputs of modern pretrained chest X-ray segmentation neural network models, which provide accurate delineations of anatomical structures such as the lungs and heart. Such models enable the extraction of interpretable geometric and volumetric features, including organ areas, relative volumes, and clinically meaningful ratios, that are routinely used in practice and well suited for distributional evaluation. We recommend using these clinically grounded features as embedding functions within the ECS framework. In Section \ref{sec:xray}, we demonstrate this approach by using a pre-trained segmentation model to extract these features for comparison. 

\subsection{Choice of $T$}
The frequency parameter $T$ controls the scale at which discrepancies between distributions are probed. As discussed in Section~\ref{sec:meth}, evaluating characteristic functions near the origin emphasizes low-frequency behavior, which reflects coarse distributional features and is closely tied to moment and tail properties. Smaller values of $T$ place greater emphasis on global structure and tail behavior. In practice, we recommend evaluating ECS over a range of $T$ values near zero rather than relying on a single choice, providing a more robust picture of distributional differences.

Importantly, ECS is not intended to be optimized over choices of $f$ or $T$. We view ECS as a diagnostic tool whose value lies in comparative evaluation across models, rather than as a criterion to be tuned to achieve favorable outcomes. 

\subsection{Resampling-based calibration of ECS}

Because ECS is a continuous discrepancy measure, its magnitude is most meaningful when interpreted relative to a suitable reference. We therefore propose a simple resampling-based calibration procedure that contextualizes the scale of ECS values. The main idea is to compute the ECS between patient images to quantify intrinsic variability. ECS values comparing patient and synthetic images are then interpreted relative to this patient–patient baseline.

Recall that ECS compares a collection of $n$ patient images $\{X_i\}_{i=1}^n$ independently sampled from the target distribution $P$ to a collection of $m$ synthetic images $\{\tilde X_j\}_{j=1}^m$ under a fixed embedding map $f$ and frequency parameter $T$. We can construct an empirical reference distribution for ECS by repeatedly drawing two subsets of sizes $n$ and $m$ from the patient images $\{X_i\}_{i=1}^n$ with replacement and computing the ECS between these two subsets. Repeating this procedure $B$ times yields a collection of ECS values $
\left\{ \hat{r}_{f,T}^{(b)}(P,P) \right\}_{b=1}^B$, which characterizes the variability of ECS arising from finite-sample effects under the target distribution. We may then compute the ECS using $n$ patient and $m$ synthetic images $\hat{r}_{f,T}(P,\tilde P)$ and compare to the resampled empirical reference $\{ \hat{r}_{f,T}^{(b)}(P,P) \}_{b=1}^B$. For example, one may report the empirical quantile of $r_{f,T}(P,\tilde P)$ within the reference distribution, or report its ratio relative to the median of the resampled distribution. 

This resampling-based calibration provides an interpretable scale for ECS by comparing discrepancies between patient and synthetic images to the variability observed in patient–patient comparisons. Importantly, while this resampling procedure is similar to the bootstrap \citep{efron2000bootstrap}, the goal here is not to compute confidence intervals nor perform hypothesis tests. Instead, we use this resampling procedure to aid interpretation and comparison of ECS values.

\section{Simulations and Empirical Studies to Evaluate ECS}
\label{sec:empirical}
Here, we evaluate ECS empirically. We divide our empirical studies into three parts. First, we assess the ability of ECS to capture tail and higher order moment information in a simulation, where we are able to precisely control the tail behavior of the underlying distributions. 
Next, we compare the behavior of ECS versus FID on the standard benchmark datasets CIFAR10 and MNIST using Inception v3 embeddings. 
Finally, we demonstrate the use of ECS in a large-scale chest X-ray dataset using clinically relevant embeddings. 

\subsection{ECS on synthetic features}
To evaluate the ability of ECS to capture tail information, we simulate random features from two classes of distributions -- multivariate normal distributions and multivariate $t$ distributions. We choose the multivariate $t$ distribution because it allows us to precisely control the tail and moment behavior as we vary the degree-of-freedom parameter $\text{df}$. As $\text{df}$ becomes large, the multivariate $t$ distribution has lighter tails and more higher order moments defined. On the other hand, as $\text{df}$ becomes smaller, the distribution becomes more heavy-tailed and exhibits more outliers, with the second moment undefined when $\text{df} \leq 2$ and first moment undefined when $\text{df} \leq 1$. A visualization is provided in Figure \ref{fig:pca} in the Appendix. 

We draw $1,000,000$ independent samples from a multivariate normal distribution centered at the origin with an identity covariance matrix with dimension $p = 32$. We then draw $1,000,000$ independent samples from five different $32$-dimensional multivariate $t$ distributions with different degree-of-freedom parameters $(100, 10, 5, 3, 2.01)$. These multivariate $t$ distributions are all chosen to be centered at the origin with an identity covariance matrix, which is achieved by choosing the scale matrix to be $\frac{\text{df}-2}{\text{df}}I_{p}$, where $I_p$ is the $p\times p$ identity matrix. We then compared samples from the multivariate normal distribution with samples from each of these five multivariate $t$ distributions using ECS at $T = 1$ and $T = 0.5$. We report all scores as normalized by the dimension to facilitate comparison.

The above process is repeated five times. We find that, as the degree-of-freedom of the multivariate $t$ distribution decreases towards $2$, which implies heavier tails, the value of the ECS increases (Table \ref{table:compare}). This trend demonstrates that as the tails of the $t$-distribution deviate further from the Gaussian baseline, the ECS captures the widening discrepancy between the two distributions. On the other hand, under the normality assumptions in the FID, the population Fr\'echet distance yields $0$ for all five comparisons, since it only depends on the mean and covariance parameters of the distributions under comparison, which are set to be identical in the simulation. These simulations highlight the need to look beyond the first and second moments when evaluating generative models.  

\begin{table}  
\caption{Mean ECS values $\pm$ standard errors between multivariate normal and $t$ distributions with varying degrees-of-freedom (estimates rounded to 3 decimal places) across five iterations. Note that the population Fr\'echet distance yields $0$ for all of these comparisons.} 
\label{table:compare}
\begin{center}
\begin{tabular}{l|rrr}
& ECS ($T = 1$) & ECS ($T = 0.5$) \\\hline
Normal versus t with df = 100 & $0.002\pm 5\times 10^{-5}$ & $0.001 \pm 5 \times 10^{-5}$  \\
Normal versus t with df = 10 & $0.020 \pm 5\times 10^{-5}$ & $0.004 \pm 4 \times 10^{-5}$  \\
Normal versus t with df = 5 &  $0.054 \pm 7 \times 10^{-5}$ & $0.015 \pm 3\times 10^{-5}$  \\
Normal versus t with df = 3 & $0.129 \pm 1 \times 10^{-4}$ & $0.055 \pm 1 \times 10^{-4}$  \\
Normal versus t with df = 2.01 & $0.379 \pm 4 \times 10^{-5}$ & $0.226 \pm 4 \times 10^{-5}$ \\
\end{tabular}
\end{center}
\end{table}

\subsection{Comparing ECS and FID on natural images}
We now compare the behavior of ECS versus FID on natural images via the CIFAR10 \citep{krizhevsky2009learning} and MNIST \citep{lecun1998gradient} datasets with Inception v3 embeddings. The goal is to assess how ECS and FID behave on standard benchmark datasets. 

The Inception v3 convolutional neural network \citep{szegedy2016rethinking} is a classifier that achieves $21.2$ top-1 and $5.6$ top-5 classification error on ImageNet ILSVRC 2012, a standard benchmark that contains images from $1000$ categories \citep{deng2009imagenet}. This suggests that the last layer embeddings computed by the Inception v3 model capture relevant information about natural images. As a result, Inception v3 embeddings are widely used for evaluating generative models and synthetic natural images. 

The ECS approach can work with any available embeddings. To conduct a fair comparison between FID and ECS on natural images, we use Inception v3 embeddings. To generate synthetic data, we use a deep convolutional generative adversarial network (DC-GAN) \citep{radford2016unsupervised} that is pretrained on CIFAR10 and MNIST data. Implementation details of the DC-GAN model can be found in the Appendix. We generate $1000$ synthetic images using this model. We also randomly select $1000$ images from each of the CIFAR10 and MNIST testing datasets. A sample of these synthetic and source images are shown in Figure \ref{fig:image}. 

\begin{figure}[t!]
    \centering
    \begin{subfigure}[t]{0.5\textwidth}
        \centering
        \includegraphics[height=2.3in]{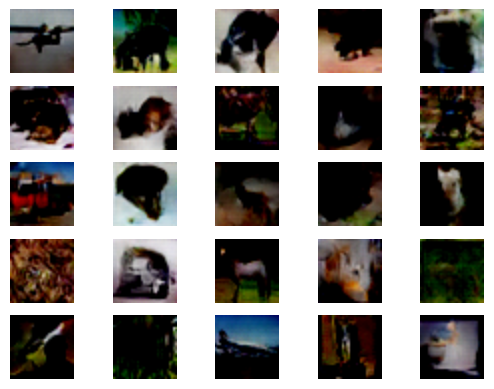}
        \caption{}
    \end{subfigure}%
    ~ 
    \begin{subfigure}[t]{0.5\textwidth}
        \centering
        \includegraphics[height=2.3in]{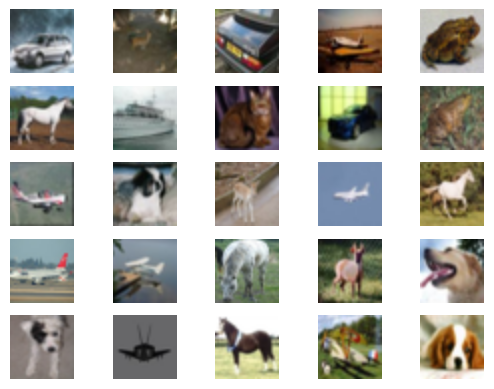}
        \caption{}
    \end{subfigure}

        \begin{subfigure}[t]{0.5\textwidth}
        \centering
        \includegraphics[height=2.3in]{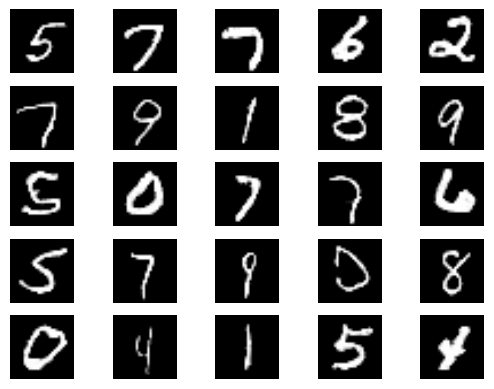}
        \caption{}
    \end{subfigure}%
    ~ 
    \begin{subfigure}[t]{0.5\textwidth}
        \centering
        \includegraphics[height=2.3in]{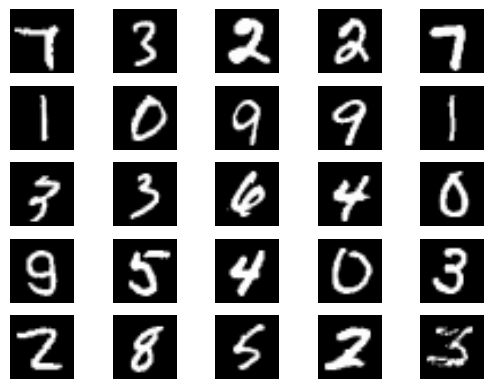}
        \caption{}
    \end{subfigure}
    \caption{\textbf{Sample of synthetic and real images.} (a) 25 synthetic images generated from a DC-GAN model pretrained on CIFAR10 data. (b) 25 source images from the CIFAR10 testing dataset. (c) 25 synthetic images generated from a DC-GAN model pretrained on MNIST data. (d) 25 source images from the MNIST testing dataset.}
    \label{fig:image}
\end{figure}

We then computed $1000$-dimensional embeddings over these source and synthetic images using the pretrained Inception v3 model available in PyTorch \citep{paszke2019pytorch, marcel2010torchvision}. We first assess whether these embeddings collectively satisfy the normality assumption in the FID approach. We perform two statistical tests on the embeddings: the Mardia test \citep{mardia1970measures}, which is based on kurtosis statistics, and the Henze-Zirkler test \citep{henze1990class}, which assesses normality based on the empirical characteristic function of the residuals (Table \ref{table:test}). We report all p-values smaller than the machine precision value of $2.2 \times 10^{-16}$ as that value. The results of these two tests present clear evidence that the Inception embeddings for both source images and synthetic images are highly non-normal across both the MNIST and CIFAR10 datasets (Table \ref{table:test}). This suggests that the normality assumptions that underlie the FID are not always empirically valid, which is in line with the conclusions of prior work \citep{jayasumana2024rethinking}.  

\begin{table}[h!] 

\caption{Multivariate normality tests for embeddings of real and synthetic images \label{table:test}}
\begin{center}
\begin{tabular}{l|c|c}
&  Mardia Kurtosis Test  & Henze-Zirkler Test \\\hline
\makecell{CIFAR10 Embeddings \\ (Real Images)} & \makecell{Test Statistic $= -1725.478$ \\ p-value $= 2.2 \times 10^{-16}$}  & \makecell{Test Statistic $= 4000$ \\ p-value $= 2.2 \times 10^{-16}$}   \\\hline
\makecell{CIFAR10 Embeddings \\ (Synthetic Images)} & \makecell{Test Statistic $= -10408.918$ \\ p-value $= 2.2 \times 10^{-16}$}   & \makecell{Test Statistic $= 4000$ \\ p-value $= 2.2 \times 10^{-16}$}  \\\hline
\makecell{MNIST Embeddings \\ (Real Images)} & \makecell{Test Statistic $= 811.932$ \\ p-value $= 2.2 \times 10^{-16}$}   & \makecell{Test Statistic $= 4000$ \\ p-value $= 2.2 \times 10^{-16}$}   \\\hline
\makecell{MNIST Embeddings \\ (Synthetic Images)} & \makecell{Test Statistic $= 83693.154$ \\ p-value $= 2.2 \times 10^{-16}$}   & \makecell{Test Statistic $= 4000$ \\ p-value $= 2.2 \times 10^{-16}$}  \\

\end{tabular}
\end{center}
\end{table}

Next, using the same Inception features, we computed the FID and ECS values (for $T = 1, 0.5, 0.1$) between the source and synthetic images of both the MNIST and CIFAR10 datasets (Table \ref{table:eval_results}). We also performed the resampling calibration procedure for $50$ resampling iterations to obtain a reference distribution that captures variability within the source images. We report both the ECS and FID values that we computed, as well as the ratio between the ECS and FID values against the median of their corresponding resampling reference distribution. 

We observe that both FID and ECS scores are higher for CIFAR than MNIST. This is in line with prior literature \citep{wei2022duelgan, miyato2018spectral, liu2023spiking} and likely due to the fact that CIFAR10 is a more complex dataset than MNIST. We also observe that the ECS demonstrates a more pronounced deviation from the reference resampling distribution compared to FID scores. This suggests that the ECS is sensitive to distributional discrepancies, specifically in higher-order moments and tails, that are otherwise obscured by the multivariate normality assumption inherent in the FID calculation. 

\begin{table}[h!] 

\caption{FID and ECS values for real versus synthetic images for both the MNIST and CIFAR10 datasets \label{table:eval_results}}
\begin{center}
\begin{tabular}{l|c|c}
&  MNIST & CIFAR10     \\\hline
ECS ($T = 1$) &  $0.222$ & $0.285$   \\\hline
ECS ($T = 0.5$) &  $0.262$ & $0.402$  \\\hline
ECS ($T = 0.1$) &  $0.279$ & $0.458$  \\\hline
FID per dimension &  $0.016$ & $0.026$ \\\hline
ECS ($T = 1$)/median of reference &  $10.458$ & $9.950$   \\\hline
ECS ($T = 0.5$)/median of reference  &  $11.714$ & $11.620$  \\\hline
ECS ($T = 0.1$)/median of reference  &  $12.539$ & $12.817$  \\\hline
FID per dimension/median of reference &  $3.783$ & $3.078$ \\\hline
\end{tabular}
\end{center}
\end{table}
\section{Chest X-Ray Application}
\label{sec:xray}
Next, we study the performance of ECS on a large-scale chest X-ray dataset. Patient chest X-ray images are obtained from the publicly available ChestX-ray14 database \citep{wang2017chestx}. Synthetic images are generated using a large, progressively growing generative adversarial network (PGGAN) designed for medical imaging known as MediGAN \citep{osuala2023medigan}. Importantly, the MediGAN model is pre-trained on the same ChestX-ray14 database (Figure \ref{fig:cxr_panel}), providing a controlled setting for our evaluation. 

Our goal is to evaluate whether ECS can detect distributional discrepancies between real and synthetic images in a clinically relevant setting. To extract clinically meaningful information for evaluation, we use a feature embedding map specifically designed for chest X-ray data. Concretely, we use TorchXRayVision \citep{cohen2022torchxrayvision}, a publicly available Python package that provides accurate pre-trained models for segmentation of anatomical structures in chest X-ray images. From these segmentations, we extract features corresponding to clinically important quantities, including the volumes of major organs such as the heart and the left--right lungs, as well as ratios such as the CTR that is commonly used in clinical practice. A list of features are included in the Table \ref{tab:cxr_labels} in the Appendix for reference. 

We compute ECS values between patient and synthetic chest X-ray images using these clinically informed features for $T = 1, 0.5, 0.1$. We also report the ratio between the computed ECS values and the median of the reference resampling distribution. We compare against the Fr\'echet distance (FD), which has the same structure as FID but with Inception embedding swapped out for the clinically relevant embedding (Table \ref{table:eval_results_CXR}).  
\begin{table}[h!] 
\caption{FD and ECS values for patient versus synthetic chest X-ray images. \label{table:eval_results_CXR}}
\begin{center}
\begin{tabular}{l|c|c}
&  Absolute Results &  Ratio to Reference Median    \\\hline
ECS ($T = 1$) &  $0.021$ & $5.068$   \\\hline
ECS ($T = 0.5$) &  $0.021$ & $5.002$  \\\hline
ECS ($T = 0.1$) &  $0.021$ & $5.063$  \\\hline
FD per dimension &  $0.003$ & $2.167$ \\\hline
\end{tabular}
\end{center}
\end{table}

We observe that, again, the ECS scores are relatively further from the center of the reference distribution than the Fr\'echet distances (Figure \ref{fig:ecs_bootstrap_panel}). This suggests that there are differences in higher order moments and tails between the patient and synthetic CXR data using segmentation and feature extraction to summarize the data. Perhaps more importantly, due to the interpretable nature of the embedding used, we can visualize the differences in patient versus synthetic distributions of particular features that underlie the ECS scores. We compare histograms of the features representing the cardiothoracic ratio (CTR) and the left--right lung ratio (Figure \ref{fig:ctr_lungratio}). In the case of left--right lung ratios, we observe that, while the synthetic data matches the center of the real data distribution closely, the real distribution has substantially higher variability---including heavier tails and outliers---that the synthetic data did not capture. For CTR, we observe that the synthetic data distribution did not match the real data distribution well, not only missing the outliers and heavy tails, but also missing the center of the distribution.  

Both the CTR and the left–right lung ratio are clinically meaningful quantities that play an important role in diagnostic assessment \citep{fraser1990diagnosis}. Deviations in the left–right lung ratio may reflect asymmetry caused by conditions such as pneumonia, atelectasis, or pleural effusion, while an elevated CTR is commonly used as a screening indicator for cardiomegaly and heart failure. The discrepancies detected by ECS in these features therefore suggest that the synthetic data may under-represent extreme or atypical presentations that could indicate pathological conditions. 

\begin{figure}[t]
\centering
\includegraphics[width=0.9\textwidth]{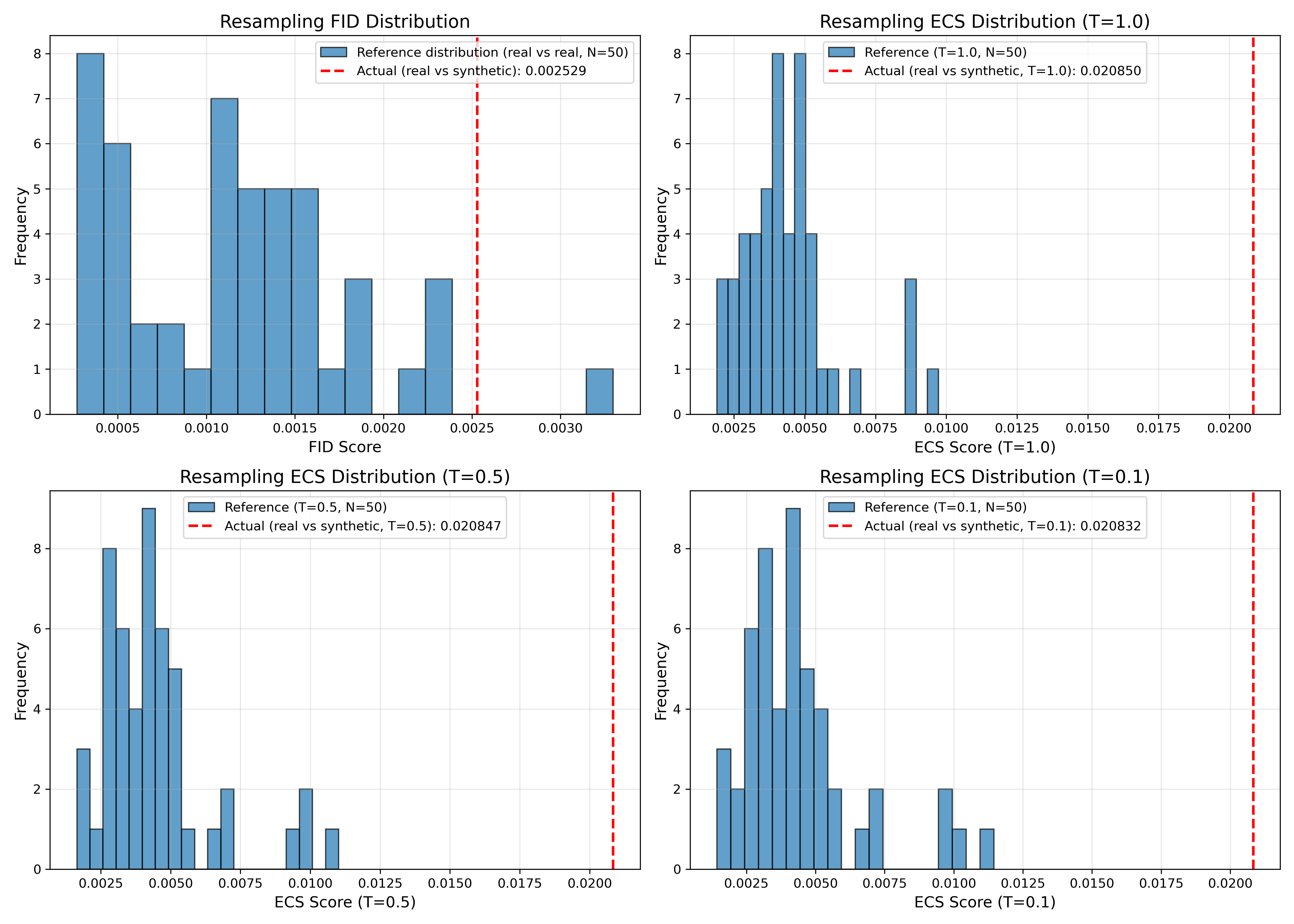}
\caption{\textbf{Resampling-based calibration of ECS for chest X-ray images.}
Comparing the reference distribution of ECS values obtained from resampled real data with the ECS value computed between patient and synthetic images (red vertical line).}
\label{fig:ecs_bootstrap_panel}
\end{figure}

\begin{figure}[t]
\centering
\includegraphics[width=0.9\textwidth]{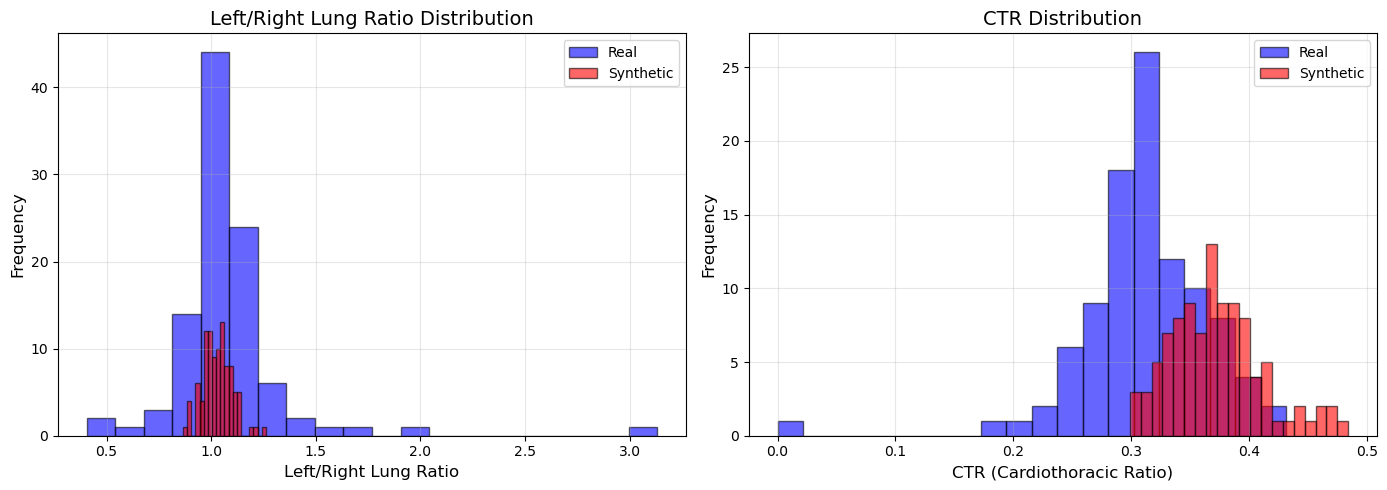}
\caption{\textbf{Comparison of clinically relevant features between real and synthetic chest X-ray images.}
Empirical distributions of the cardiothoracic ratio (CTR) and the left--right lung volume ratio computed from patient chest X-ray images (blue) and synthetic images (red). }
\label{fig:ctr_lungratio}
\end{figure}

\section{Discussion}
\label{sec:discussion}

A core aspect of evaluating synthetic medical images lies in how the problem is framed. Synthetic medical images are typically explicitly tagged as such, so the objective is not to classify them as authentic or synthetic. It is also understood that synthetic and patient images necessarily differ at a distributional level, given that they are generated through fundamentally different processes. Accordingly, the goal is not to conduct a binary hypothesis test of distributional equality. Rather than asking whether statistically significant differences exist, the more meaningful question is how those differences manifest along clinically important dimensions, particularly in features relevant for diagnosis and downstream use. This perspective motivates the use of a continuous discrepancy measure such as ECS. 

Our empirical study on chest X-ray images highlights the value of combining clinically meaningful and interpretable embeddings with distributional evaluation methods that are sensitive to tail behavior. By focusing on clinically relevant features such as the CTR and the left--right lung volume ratio, we uncovered substantial discrepancies in the tails and variability of these features between patient and synthetic CXR images. 

These findings are particularly important in medical imaging, where pathological cases often lie in the tails of the feature distribution rather than near its center. Synthetic images that fail to capture this tail behavior may under-represent important cases, with potentially adverse consequences for downstream model development and evaluation. Our results underscore the need for evaluation tools explicitly assess distributional properties of synthetic medical images.

The perception that deep generative models are highly expressive has encouraged practitioners to use synthetic data from these powerful models to substitute or supplement authentic data in downstream applications, particularly in medical imaging. This belief is grounded in the universal approximation properties of deep neural networks \citep{cybenko1989approximation, barron1993universal, hornik1991approximation} and reinforced by the broader generative modeling literature \citep{doersch2016tutorial, kingma2019introduction}.

However, recent work shows that broad classes of commonly used deep generative models are structurally limited to learning light-tailed distributions \citep{tam2024On}. When the true data-generating distribution is heavy-tailed, this limitation may lead to systematic under-representation of extreme or atypical cases. In clinical settings, such tail observations often correspond to pathological conditions or severe disease presentations. Because commonly used evaluation metrics emphasize central distributional features while largely ignoring tail behavior, these clinically important discrepancies may go undetected in practice when using tools such as Fr\'echet distances. A related implication of our findings is the critical importance of developing generative models that can more accurately capture extreme and tail behavior in medical applications.  

While we focused on evaluating image generative models in this work, the approach presented here can in principle be applied to other data modalities, as long as an appropriate domain-specific embedding map $f: \X \to \R^p$ is available. Developing similar evaluation metrics for synthetic data from more complex generative tasks, such as conditional generation and multi-modal data generation, remains an open challenge.

\bigskip
\begin{center}
{\large\bf Appendix}
\end{center}
\section{Code Availability}
All code related to this study can be found at \url{https://github.com/edrictam/Embedded-Characteristic-Score}.
\section{Proofs}
\subsection{Proof of Theorem \ref{thm:converge}}
\begin{proof}

Recall that the characteristic function of any real-valued random variable always exists, and is always bounded. By the weak law of large numbers, 
$$\frac{1}{n}\sum_{i = 1}^n \exp(\mi T f_\rho(X_i)) \to_{p} E(\exp(\mi T f_\rho(X)))$$ as $n\to\infty$ and 
$$\frac{1}{m}\sum_{j = 1}^m \exp(\mi T f_\rho(\tilde{X}_j)) \to_{p} E(\exp(\mi T f_\rho(\tilde{X})))$$ as $m\to\infty$, where $\to_p$ denotes convergence in probability.   

This implies \citep{resnick2013probability} that 
\begin{align*}
\frac{1}{n}\sum_{i = 1}^n \exp(\mi T f_\rho(X_i)) - \frac{1}{m}\sum_{j = 1}^m \exp(\mi T f_\rho(\tilde{X}_j)) \\ \to_{p} E(\exp(\mi T f_\rho(X))) - E(\exp(\mi T f_\rho(\tilde{X})))
\end{align*} as $n, m \to \infty$.

By the continuity of the norm $||\cdot||_2$ and the continuous mapping theorem, 
\begin{align*}
||\frac{1}{n}\sum_{i = 1}^n \exp(\mi T f_\rho(X_i)) - \frac{1}{m}\sum_{j = 1}^m \exp(\mi T f_\rho(\tilde{X}_j))||_2  \\ \to_{p} ||E(\exp(\mi T f_\rho(X))) - E(\exp(\mi T f_\rho(\tilde{X})))||_2
\end{align*} as $n, m \to \infty$.
Noting that the above holds for any $\rho$ in $1, \cdots, p$, again this implies \citep{resnick2013probability} that
\begin{align*}\sum_{\rho = 1}^p||\frac{1}{n}\sum_{i = 1}^n \exp(\mi T f_\rho(X_i)) - \frac{1}{m}\sum_{j = 1}^m \exp(\mi T f_\rho(\tilde{X}_j))||_2\\ \to_{p} \sum_{\rho = 1}^p||E(\exp(\mi T f_\rho(X))) - E(\exp(\mi T f_\rho(\tilde{X})))||_2\;.
\end{align*}
Applying the continuous mapping theorem again, we get
\begin{align*}\frac{1}{pT}\sum_{\rho = 1}^p||\frac{1}{n}\sum_{i = 1}^n \exp(\mi T f_\rho(X_i)) - \frac{1}{m}\sum_{j = 1}^m \exp(\mi T f_\rho(\tilde{X}_j))||_2\\ \to_{p} \frac{1}{pT}\sum_{\rho = 1}^p||E(\exp(\mi T f_\rho(X))) - E(\exp(\mi T f_\rho(\tilde{X})))||_2\;,
\end{align*}
which is equivalent to $\hat{r}_{f, T}(P, \tilde{P})$ converging in probability to $r_{f, T}(P, \tilde{P})$ as $n, m \to \infty$. 
\end{proof}
\subsection{Proof of Theorem \ref{thm:pseudo}}
\begin{proof}
This is immediate from the definition that $r_{f, T}(P, \tilde{P}) \geq 0$ and $r_{f, T}(P, \tilde{P}) = r_{f, T}(\tilde{P}, P)$ for any $P, \tilde{P}$. If $P = \tilde{P}$, then \begin{align*}
    r_{f, T}(P, \tilde{P}) &= \frac{1}{pT}\sum_{\rho = 1}^p||E\exp(\mi T f_\rho(X)) - E\exp(\mi T f_\rho(\tilde{X}))||_2 \\
    &= \frac{1}{pT}\sum_{\rho = 1}^p||E\exp(\mi T f_\rho(X)) - E\exp(\mi T f_\rho(X))||_2 \\
    &= 0 \; .
\end{align*} 
Consider any three distributions $P, \tilde{P}, P^*$. 
\begin{align*}
r_{f, T}(P, \tilde{P}) + r_{f, T}(\tilde{P}, P^*)
&= \frac{1}{pT}\sum_{\rho = 1}^p||E\exp(\mi T f_\rho(X)) - E\exp(\mi T f_\rho(\tilde{X}))||_2 \\ &+ \frac{1}{pT}\sum_{\rho = 1}^p ||E\exp(\mi T f_\rho(\tilde{X})) - E\exp(\mi T f_\rho(X^*))||_2 \\
&= \frac{1}{pT}\sum_{\rho = 1}^p[||E\exp(\mi T f_\rho(X)) - E\exp(\mi T f_\rho(\tilde{X}))||_2 \\ &+ ||E\exp(\mi T f_\rho(\tilde{X})) - E\exp(\mi T f_\rho(X^*))||_2 ]\\
&\geq \frac{1}{pT}\sum_{\rho = 1}^p ||E\exp(\mi T f_\rho(X)) - E\exp(\mi T f_\rho(X^*))||_2 \\ &= r_{f, T}(P, P^*) \;.
\end{align*}
Here the last inequality results from the triangle inequality property of $||\cdot||_2$.
Thus, the triangle inequality holds for $r_{f, \T}(\cdot, \cdot)$. 
We have therefore shown that $r_{f, \T}(\cdot, \cdot)$ satisfies all the properties required for a pseudometric. 
\end{proof}

\section{Implementation details of empirical study}

We conducted the entire empirical study in Python and PyTorch \citep{paszke2019pytorch}. We conducted the Mardia Kurtosis test using a function implementation referenced from the open source repository \url{https://github.com/quillan86/mvn-python}. We used a function from the package Pingouin \citep{Vallat2018Pingouin} to perform the Henze-Zirkler test. The p-values returned from these tests are too small to be represented by double precision floating points, hence we report in Table \ref{table:test} the machine epsilon value of $2.2\times 10^{-16}$.

Our implementation of the DC-GAN model \citep{radford2016unsupervised} is directly referenced from the open source repository \url{https://github.com/csinva/gan-vae-pretrained-pytorch}. The pretrained weights for the DC-GAN models under both the MNIST and CIFAR10 datasets are also referenced from the same repository. The DC-GAN model architecture that is adopted in our implementation consists of $5$ 2D convolution layers in both the discriminator and generator networks. 

In the empirical study, we obtained CIFAR10 \citep{krizhevsky2009learning} and MNIST \citep{lecun1998gradient} as built-in datasets from the Torchvision package\citep{marcel2010torchvision}. We use the built-in pretrained Inception v3 model \citep{szegedy2016rethinking} from the Torchvision package to compute the Inception embeddings. 

\section{PCA visualization of Multivariate $t$ versus Gaussian samples}

In our simulations in Section \ref{sec:empirical}, we used the ECS to compare random features from a standard multivariate Gaussian distribution to random features from multivariate $t$ distributions with varying degree-of-freedom parameters $\text{df} = \{100, 10, 5, 3, 2.1\}$. Here, we visualize the differences in tail behavior between these distributions. We draw $10000$ independent random vectors with dimension $p = 32$ from each of these $6$ distributions. We then visually compare the Gaussian vectors with each of the $5$ groups of multivariate $t$ vectors via two-dimensional PCA plots (Figure \ref{fig:pca}). We also perform a baseline Gaussian-to-Gaussian comparison by visualizing the original $10000$ Gaussian samples alongside another $10000$ Gaussian samples. 

The visualizations indicate that, while the multivariate $t$ distribution with $\text{df} = 100$ behaves very similarly to a Gaussian distribution, as we decrease the degree-of-freedom parameter the multivariate $t$ distribution exhibits more outliers and heavier tails, whereas the Gaussian samples exhibited no outliers. 

\begin{figure}[t!]
    \centering
    \begin{subfigure}[t]{0.5\textwidth}
        \centering
        \includegraphics[height=1.8in]{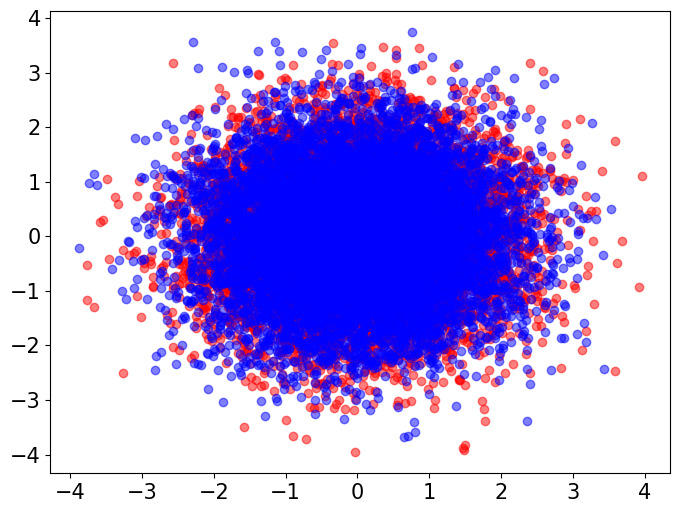}
        \caption{}
    \end{subfigure}%
    ~ 
    \begin{subfigure}[t]{0.5\textwidth}
        \centering
        \includegraphics[height=1.8in]{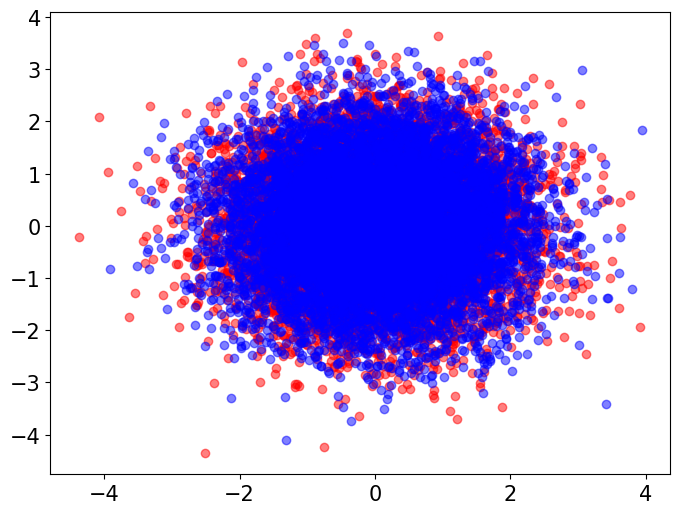}
        \caption{}
    \end{subfigure}

        \begin{subfigure}[t]{0.5\textwidth}
        \centering
        \includegraphics[height=1.8in]{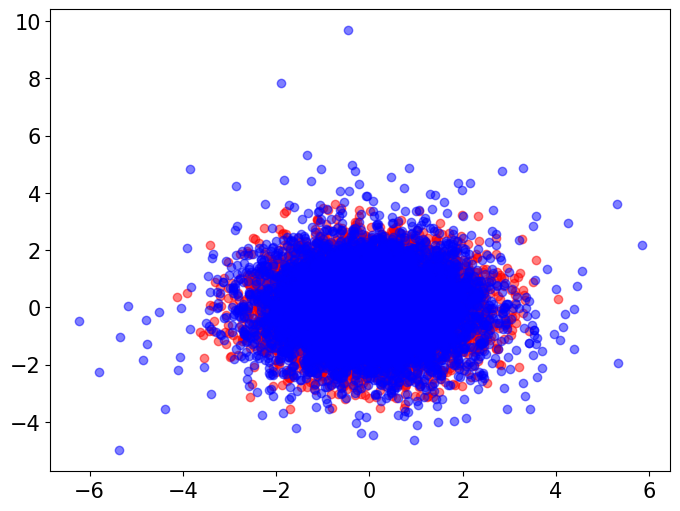}
        \caption{}
    \end{subfigure}%
    ~ 
    \begin{subfigure}[t]{0.5\textwidth}
        \centering
        \includegraphics[height=1.8in]{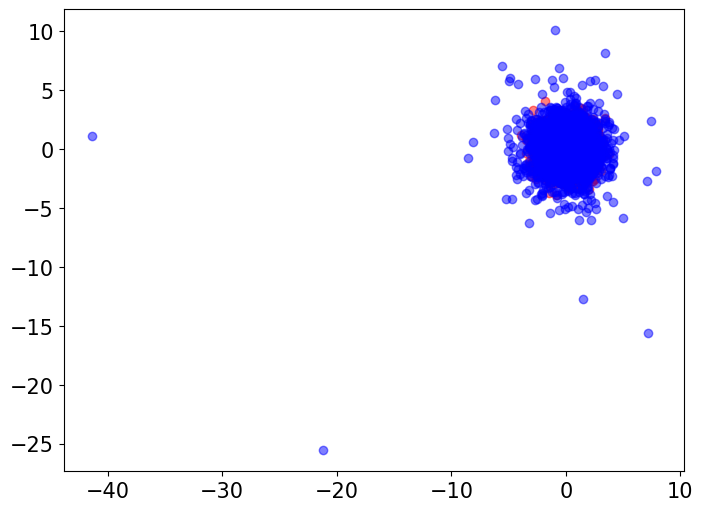}
        \caption{}
    \end{subfigure}

        \begin{subfigure}[t]{0.5\textwidth}
        \centering
        \includegraphics[height=1.8in]{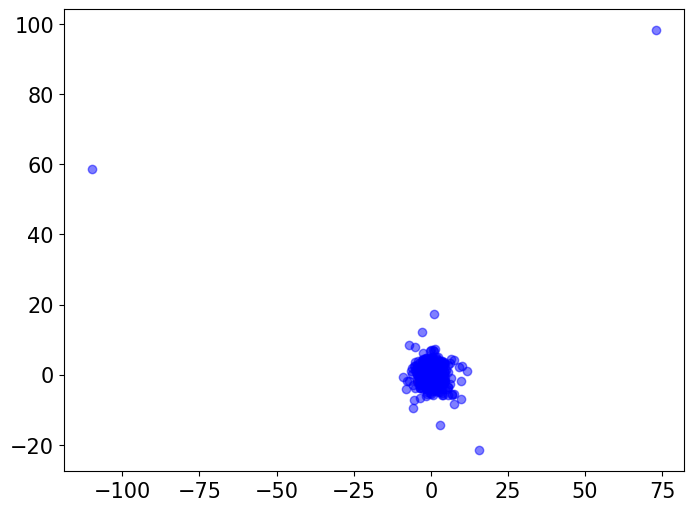}
        \caption{}
    \end{subfigure}%
    ~ 
    \begin{subfigure}[t]{0.5\textwidth}
        \centering
        \includegraphics[height=1.8in]{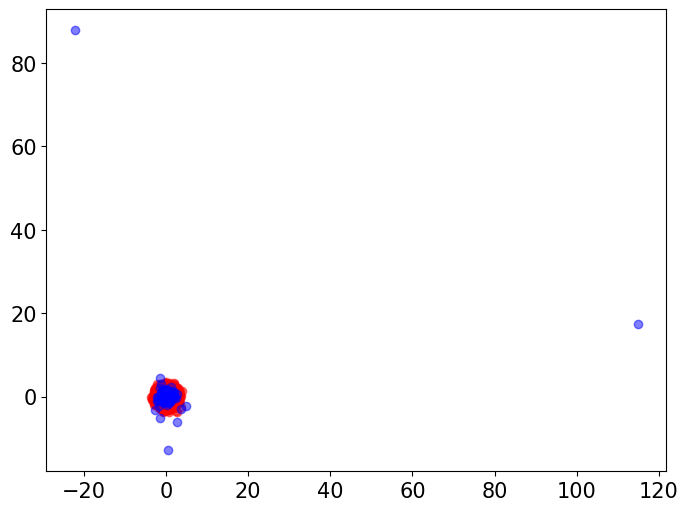}
        \caption{}
    \end{subfigure}

    \caption{\textbf{Comparing Gaussian samples versus multivariate $t$ samples via two dimensional PCA plots.} (a) Gaussian (red) versus Gaussian (blue).   (b) Gaussian (red) versus Multivariate $t$ ($\text{df} = 100$) (blue).
    (c) Gaussian (red) versus Multivariate $t$ ($\text{df} = 10$) (blue).
    (d) Gaussian (red) versus Multivariate $t$ ($\text{df} = 5$) (blue).
    (e) Gaussian (red) versus Multivariate $t$ ($\text{df} = 3$) (blue).
    (f) Gaussian (red) versus Multivariate $t$ ($\text{df} = 2.01$) (blue).}
    \label{fig:pca}
\end{figure}

\subsection{Clinically Relevant Features for Chest X-Ray Data}
Below are the list of features extracted by the pre-trained CXR segmentation pipeline. 
\begin{table}[H]
\centering
\caption{Anatomical structures identified in chest X-ray segmentation.}
\label{tab:cxr_labels}
\begin{tabular}{cl}
\toprule
Index & Label \\
\midrule
0  & Left Clavicle \\
1  & Right Clavicle \\
2  & Left Scapula \\
3  & Right Scapula \\
4  & Left Lung \\
5  & Right Lung \\
6  & Left Hilus Pulmonis \\
7  & Right Hilus Pulmonis \\
8  & Heart \\
9  & Aorta \\
10 & Facies Diaphragmatica \\
11 & Mediastinum \\
12 & Trachea \\
13 & Spine \\
14 & Cardiothoracic Ratio (CTR) \\ 
15 & Left--right Lung Ratio \\
\bottomrule
\end{tabular}
\end{table}
\begin{acks}[Acknowledgments]
\end{acks}
\begin{funding}
ET and BEE were funded in part by grants from the Parker Institute for Cancer Immunology (PICI), the Chan-Zuckerberg Institute (CZI), the Biswas Family Foundation, NIH NHGRI R01 HG012967, and NIH NHGRI R01 HG013736. ET's research has been supported by a Warren Alpert Fellowship and a Croucher Fellowship. 
BEE is a CIFAR Fellow in the Multiscale Human Program.

BEE is on the Scientific Advisory Board for ArrePath Inc, GSK AI for Cancer, and Freenome.
\end{funding}


\newpage
\bibliographystyle{imsart-nameyear} 
\bibliography{references}

\end{document}